\RequirePackage{fix-cm}
\documentclass[smallextended]{svjour3}       % onecolumn (second format)
\smartqed  % flush right qed marks, e.g. at end of proof
\usepackage{graphicx}
\usepackage[utf8]{inputenc}
\usepackage[T1]{fontenc}
\usepackage{subcaption}
\usepackage{url}
\usepackage[boxed,linesnumbered]{algorithm2e}
\usepackage{listings}
\usepackage{mathtools}
\usepackage{amsfonts}
\usepackage{amssymb}

\newcommand{\bcp}{\(BC+\)}
\newcommand{\cpasp}{\texttt{CPLUS2ASP}}

\newcommand{\iclingo}{\texttt{iClingo}}

\newcommand{\gup}{\emph{go up}}
\newcommand{\gdown}{\emph{go down}}
\newcommand{\gleft}{\emph{go left}}
\newcommand{\gright}{\emph{go right}}

\newcommand{\ASPRL}{ASP(RL)}
\newcommand{\ASPSARSA}{ASP(SARSA)}
\newcommand{\ASPQ}{ASP(Q-Learning)}

\usepackage{dsfont}
\renewcommand{\Re}{\mathds{R}}

\newcommand{\ssq}{\subseteq}

\newcommand{\set}{\mathcal}

\newcommand{\state}{s}
\newcommand{\fstate}{\state'}
\newcommand{\istate}{\state_0}
\newcommand{\gstate}{\state_g}
\newcommand{\action}{a}
\newcommand{\trans}{t}
\newcommand{\reward}{r}
\newcommand{\pol}{\pi}
\newcommand{\opt}{\ast}
\newcommand{\optpol}{\pol^{\opt}}
\newcommand{\timeinst}{t}
\newcommand{\traj}{T}

\newcommand{\lp}{\Pi}

\newcommand{\atom}{A}
\newcommand{\lit}{L}

\newcommand{\tr}{true}
\newcommand{\fl}{\mathit{false}}
\newcommand{\un}{unknown}

\newcommand{\MDP}{M}
\newcommand{\vaf}{Q}

\newcommand{\State}{\set{\MakeUppercase{\state}}}
\newcommand{\Action}{\set{\MakeUppercase{\action}}}
\newcommand{\Trans}{\set{\MakeUppercase{\trans}}}
\newcommand{\Reward}{\set{\MakeUppercase{\reward}}}
\newcommand{\Traj}{\set{H}}
\newcommand{\Empty}{\varnothing}

\newcommand{\gen}{\dot}
\newcommand{\genMDP}{\gen{\MDP}}
\newcommand{\genS}{\gen{\State}}
\newcommand{\genA}{\gen{\Action}}
\newcommand{\genT}{\gen{\Trans}}
\newcommand{\genR}{\gen{\Reward}}

\newcommand{\forb}{\overline}
\newcommand{\forbS}{\forb{\State}}
\newcommand{\forbA}{\forb{\Action}}

\newcommand{\forbQ}{\forb{\set{\vaf}(\state,\action)}}

\newcommand{\MDPdef}{\MDP = \langle \State, \Action, \Trans, \Reward\rangle}
\newcommand{\genMDPdef}{\genMDP = \langle \genS, \genA, \genT, \Reward \rangle}

\newcommand{\f}{\mathrm{F}}
\newcommand{\g}{\mathrm{G}}
\newcommand{\h}{\mathrm{H}}
\newcommand{\Desc}{\mathcal{D}}

\begin{document}

\title{Answer Set Programming for Non-Stationary Markov Decision Processes}

%\titlerunning{Short form of title}        

\author{Leonardo A. Ferreira \and Reinaldo A. C. Bianchi \and Paulo E. Santos \and Ramon Lopez de Mantaras}

\institute{Leonardo Anjoletto Ferreira \at
            Universidade Metodista de São Paulo, Rua Alfeu Tavares, 149, São Bernardo do Campo, São Paulo, Brazil.
            \email{leonardo.ferreira@metodista.br}
           \and
            Reinaldo Augusto da Costa Bianchi \at
            Centro Universitário FEI, Av. Humberto de Alencar Castelo Branco, 3972, São Bernardo do Campo, São Paulo, Brazil.
            \email{rbianchi@fei.edu.br}
            \and
            Paulo Eduardo Santos \at
            Centro Universitário FEI, Av. Humberto de Alencar Castelo Branco, 3972, São Bernardo do Campo, São Paulo, Brazil.
            \email{psantos@fei.edu.br}
            \and
            Ramon Lopez de Mantaras \at
            Institut d'Investigació en Intelligència Artificial, 08193 Bellaterra, Catalonia, Spain
            \email{mantaras@iiia.csic.es}
}

\date{Received: date / Accepted: date}

\maketitle

\begin{abstract}
Non-stationary domains, where unforeseen changes happen, present a challenge for agents to find an optimal policy for a sequential decision making problem. This work investigates a solution to this problem that combines Markov Decision Processes (MDP) and  Reinforcement Learning (RL) with Answer Set Programming (ASP) in a method we call ASP(RL). In this method, Answer Set Programming is used  to find the possible trajectories of an MDP, from where Reinforcement Learning is applied to learn the optimal policy of the problem. Results show that ASP(RL) is capable of efficiently finding the optimal solution of an MDP representing non-stationary domains.

\end{abstract}

\section{Introduction}\label{sec:intro}

John McCarthy defined Elaboration Tolerance as ``the ability to accept changes to a person's or a computer program's representation of facts about a subject without having to start all over''~\cite{McCarthy98}. An example of a real world problem that requires solutions that are tolerant to elaborations is the dynamics of urban mobility, where streets and roads are constantly reconstructed or modified. Some of these changes are planned and, thus, can be previously informed to the inhabitants of the city. However, unplanned changes due to natural phenomena (rain or snowing, for example), or due to human actions (e.g. road accidents), may occur that cause road blocks which could prevent the traffic through certain routes of the city. In such cases, it is not possible to know the changes until they are observed by the agents. However, an agent immersed in this domain must be capable of finding the best sequence of actions, considering the new situations, but without loosing all the information previously acquired.

One formalism that can be used to model the kind of situations described above is a non-stationary Markov Decision Process (MDP), where the set of states represented by observations of the environment (facts) can suffer changes over time such that states can be added to, or removed from, the decision process. As these changes may not be known {\em a priori}, the environment cannot be modelled as a stationary MDP due to the Curse of Dimensionality~\cite{Bellman-Dreyfus1962}, which describes the growth in the set of states when considering the number of variables involved in the description of a state.

This work is directed towards problem solving in non-stationary domains in which, not only the transition and the reward functions change, but also the states and actions may change during the agent's interaction with the environment. The \ASPRL{} proposed here is able to change an MDP's description during learning and to reuse the learnt data in the new domain that it is interacting with. A consequence of using \ASPRL{} is the speed up in the searching for an MDP solution as a consequence of the reduction that may occur in the search space.

In order to model an agent capable of interacting efficiently with non-stationary domains, we propose a method called \ASPRL{} that combines Markov Decision Process, Reinforcement Learning (RL) (Section~\ref{sec:sdm}) with Answer Set Programming (Section~\ref{sec:asp}). The proposed combination (Section~\ref{sec:prop}) allows an agent to learn incrementally in an environment that suffers changes. The method was analysed in a non-stationary grid world (Section~\ref{sec:exp}) and experimentally evaluated and compared to two Reinforcement Learning algorithms (Section~\ref{sec:res}).

\section{Background}\label{sec:theo}

This section introduces Markov Decision Processes (MDP), Reinforcement Learning (RL) and Answer Set Programming (ASP), which constitute the foundations of this work.

\subsection{MDP and Reinforcement Learning}\label{sec:sdm}

In a Sequential Decision Making Problem, an agent must select a series of actions in order to find a solution to a given problem. A feasible solution, known as policy (\(\pi\)), is a sequence of non-deterministic actions that leads the agent from an initial state to a goal state~\cite{Bellman1952,Bellman-Dreyfus1962}. A problem such as this may have more than one feasible solution, thus it is possible to use the Bellman's Principle of Optimality~\cite{Bellman1952,Bellman-Dreyfus1962} as a criterion to define which of the feasible policies can be considered as the optimal policy (\(\pi^\ast\)). Bellman's Principle of Optimality states that ``{\em an optimal policy has the property that whatever the initial state and initial decision are, the remaining decisions must constitute an optimal policy with regard to the state resulting from the first decision}''~\cite{Bellman-Dreyfus1962}. By this definition, an optimal policy is the one that maximises (or minimises) a desired reward/cost function.

Markov Decision Process (MDP)~\cite{Bellman1957} can be used to formalise Sequential Decision Making Problems. An MDP is defined as a tuple \(\MDP = \langle \State, \Action, \Trans, \Reward \rangle\) where:
\begin{itemize}
    \item \(\State\) is the set of states at any time step;
    \item \(\Action\) is the set of allowed actions in the states \(\state \in \State\);
    \item \(\Trans : \State \times \Action \times \State \mapsto [0,1]\) is the transition function that gives the probability of reaching the future state \(\fstate \in \State\) by performing action \(\action \in \Action\) in the current state \(\state \in \State\); 
    \item \(\Reward : \State \times \Action \times \State \mapsto \Re\) is the reward function that returns a real value for reaching a state \(\fstate \in \State\) after performing an action \(\action \in \Action\) in a state \(\state \in \State\).
\end{itemize}

To find the optimal solution of an MDP is to find, for each state, which  is the action that maximises the reward function. One of the methods that can be used to approximate such optimal solution is Reinforcement Learning (RL). With RL, at each time step, a learning agent at a state \(\state \in \State\) chooses an action \(\action \in \Action\) to be performed in the environment. After the action \(\action\) is performed, the agent receives its new state \(\fstate \in \State\) and a reward \(\reward(\state,\action,\fstate)\). This reward is used to update a value function \(V(\state)\) (or an action-value function \(\vaf(\state,\action)\), depending on the method used) and the interaction continues from the new state. Given enough time, the agent is capable to approximate the (action-) value function, maximising the reward function and finding the optimal policy. One important aspect of RL methods is that the transition and reward function are not necessarily known beforehand by the agent, but are present in the environment.

Two well-known methods of RL are SARSA~\cite{Sutton-Barto15} and Q-Learning~\cite{Watkins,Sutton-Barto15}. Both are based on the concept of updating an (action-) value function considering the observations received from the environment. The main difference between them is how this update is accomplished. SARSA is an {\em on-policy} method, which means that updates in the \(\vaf(\state,\action)\) function use the actions executed in the policy that is being followed, while Q-Learning is an {\em off-policy} method that uses the maximum value of the next state to update the current state-action pairs.

Although Reinforcement Learning allows for learning the optimal solution of a sequential decision-making problem with non-stationary transition and reward functions (functions that may change over time) and without the knowledge of the reward function, it still needs stationary sets (which do not change during the interaction) of states and actions in order to proceed with the learning process. In order to account for changes in the set of states, we propose the use of Answer Set Programming.

\subsection{Answer Set Programming}\label{sec:asp} 

Answer Set Programming (ASP) is a declarative non-monotonic logic programming language that has been used with great success to describe and provide solutions for NP-complete problems, such as planning and scheduling~\cite{khandelwal2014,yang2014}. Furthermore, ASP can be used for problems with large search space, such as the Reaction Control System of a Space Shuttle~\cite{RCS,RCS1,RCS2,RCS3}. 

An ASP program is a set of rules,  each rule is composed of an atom \(\atom\) and of literals \(\lit_m\), which are atoms or negated atoms. An ASP rule can be represented as:
%\begin{align} 
    $\atom \gets \lit_1, \ldots, \lit_n \mathrm{;}$
%    \label{eq:rule}
%\end{align}
where \(\atom\) is called the head of the rule and the conjunction of literals \(\lit_1,\ldots,\lit_n\) is its body. A rule is said to be positive when there is no negated atom in its body; when \(n=0\) the atom \(\atom\) is said to be a fact.

Let \(\lp\) be an ASP program, an answer set of \(\lp\) is an interpretation that makes all the rules of this program true. This interpretation is a minimal model of the program. One important aspect of ASP is its non-monotonic semantics (based on the Stable Model Semantics \cite{Gelfond88}), which respects the rationality principle that states that ``{\em one shall not believe anything one is not forced to believe}'' \cite{Gelfond88}. Along with \(\tr\) and \(\fl\), ASP also has a third truth value for \(\un\). 

There are two types of negation in ASP: strong (or ``classical'') and weak, which in ASP represents {\em negation as failure}~\cite{Lifschitz1999}.

Given an ASP program \(\lp\) and a set \(M\) of atoms of \(\lp\), a reduct program \(\lp_M\) is obtained from \(\lp\) by~\cite{Gelfond88}:
\begin{itemize}
    \item Deleting each rule  with a negative literal in its body in the form \(\lnot B, B \in M\);
    \item Deleting every negative literal in the body of remaining rules.
\end{itemize}

Thus, the reduct program \(\lp_M\) is negation-free and has a unique minimal Herbrand model. If \(M\) coincides with this model for \(\lp_M\), then \(M\) is a stable model of \(\lp\). Furthermore, by using an operator \(O_\lp\) defined as ``for any set of atoms \(M\) of \(\lp\), \(O_\lp(M)\) is the minimal Herbrand model of \(\lp_M\)'', then a stable model can also be described as the fixed points of \(O_\lp\). From this definition, a minimal model that accepts classical negation is called an answer set instead of a stable model.

Although ASP does not provide syntax to describe non-deterministic events, it is possible to use choice rules in order to verify each possible outcome of a choice. Considering for example that an agent is at a state \texttt{s0} and chooses to perform action \texttt{a} with the possible outcomes being the future states \texttt{s1}, \texttt{s2} and \texttt{s3}, this transition can be encoded using \texttt{ ``1 \{ s1, s2, s3 \} 1 :-  s0, a.''} in an ASP program. Thus, when \texttt{s0} and \texttt{a} are true in \(\lp\) (the agent has performed the action \texttt{a} in the state \texttt{s0}), only one of the future states \texttt{s1}, \texttt{s2} or \texttt{s3} is true (reached by the agent).

Since ASP can be used as a tool for providing reasoning and knowledge revision on a set of states and Reinforcement Learning allows for learning the solution of an MDP without the need of an explicit reward function, an opportunity arises to combine both methods in order to efficiently find the optimal policies for domains where unforeseen changes occur. The next section presents the action language \bcp{} that provides the appropriate definitions for domain modelling needed to bridge the gap between ASP and RL.

\subsection{The Action Language \bcp{}}

The action language \bcp{} is defined over the stable model semantics and allows for some useful ASP constructs, such as a high-level description of actions and their effects, as a consequence of its structured abstract representation of transition systems~\cite{BCplus2015}.

\bcp{} has two sets of symbols: action constants and fluent constants; and also two sets of formulas: fluent formula, which has only fluent constants, and action formula, which has at least one action constant and no fluent constant.

In \bcp{}, an action description is a set of causal laws that have two forms. The first is:
\begin{align}
    caused \; \f \; \mathit{if} \; \g
    \label{eq:law}
\end{align}
where, \(\f\) and \(\g\) are formulas. If \(\f\) and \(\g\) are both fluent formulas, then Formula~\ref{eq:law} is a static law. If \(\f\) is an action formula, but \(\g\) is a fluent formula, then Formula~\ref{eq:law} is an action dynamic law. The second form is called fluent dynamic law and has the form:
\begin{align}
    caused \; \f \;\mathit{if} \; \g \; \mathit{after} \; \h
    \label{eq:fdl}
\end{align}
where \(\h\) is a formula, \(\f\) and \(\g\) are fluent formulas and \(\f\) does not contain statically determined constants.

Causal dependencies between fluents in the same state are described by static laws. Direct effects of actions are represented by fluent dynamic laws, while causal dependencies between concurrently executed actions are expressed by action dynamic laws.

Given an action description \(\Desc\) expressed in \bcp{}, a stable model for the sequence \(PF_m(\Desc)\) of propositional formulae describes a path of length \(m\) in a transition system \(\Desc\)~\cite{BCplus2015}. Given a time instant \(i \in \{0, \ldots, m\}\), a translation \(PF_m(\Desc)\) is a conjunction of:
\begin{itemize}
    \item \(i : \f \gets i : \g\) for every static and atomic law in \(\Desc\) and \(\forall i \in \{0, \ldots, m - 1\}\);
    \item \(i + 1 : \f \gets (i + 1 : \g) \land (i : \h)\) for every fluent dynamic law in \(\Desc\) and \(\forall i \in \{0, \ldots, m - 1\}\);
    \item \(\{0 : c = v\}\) for every regular fluent constant \(c\) and every \(v \in Dom(c)\);
    \item Given \(\{v_1, \ldots, v_m\}\) as \(Dom(c)\), \(\bot \gets \lnot(1 \leq \{i : c = v_1, \ldots, i : c = v_m\} \leq 1)\) for every \(i : c\) representing the uniqueness of names and existence values for the constants;
\end{itemize}

The action language \bcp{} can be directly translated into an ASP program for providing sequences of actions as answer sets.

\section{Combining ASP and MDP}\label{sec:prop}
%This section uses the concepts presented above to describe the combination of ASP %and MDPs proposed in this paper, called \ASPRL{}.

This section presents the main contribution of this work, the \ASPRL{} method, which is a combination of ASP and MDP for solving non-stationary decision making problems.

\subsection{Finding the Set of States} 

In this work  Answer Set Programs, translated from \bcp{}, represent the states \(\state \in \State\), the actions \(\action \in \Action\), and the expected transition function of an MDP, along with sets \(\State_0 \in \State\) and \(\State_g \in \State\) which represents the sets of initial states and goal states respectively. Let \(\lp(\State, \Action)\) be one such ASP program with \(\State\) and \(\Action\) as set of states and actions respectively. Given an initial state \(\istate \in \State_0\) and a goal state \(\gstate \in \State_g\), an answer set of \(\lp(\State, \Action)\) represents a trajectory \(\traj\) of the form: 
\begin{align} 
\traj = \langle \langle \istate, \action_{0}, \state_{1} \rangle, \langle  
\state_{1}, \action_{1}, \state_{2} \rangle, \ldots, \langle \state_{n}, 
\action_{n}, \gstate \rangle \rangle 
\end{align}
where \(\state_{n}\) and \(\action_{n}\) are, respectively, the state and the action at time \(n\).

As ASP programs can have more than one answer set, let a set \(\Traj\) contain all trajectories \(\traj\) that represent the sequence of actions leading from an initial state to a goal state. Thus, in the set of trajectories \(\Traj\) there are a set of states visited and a set of actions performed that are subsets of those sets in the MDP defined in the logic program \(\lp(\State, \Action)\). Thus, this set \(\Traj\) can be used to describe a new MDP \(\genMDP\), as stated in the following Lemma.
\begin{lemma}\label{lemma:traj}
     Given an MDP \(\MDPdef\) described by a logic program \(\lp(\State,\Action)\), the set \(\Traj\) of trajectories found for \(\lp(\State,\Action)\) defines \(\genMDPdef\), such that \(\genMDP \ssq \MDP\). Considering that \(\genMDP \ssq \MDP\) iff \(\genS \ssq \State\) or \(\genA \ssq \Action\) or \(\genT \ssq \Trans\) or \(\genR \ssq \Reward\).
\end{lemma}

\begin{proof}(Sketch)
A logic program \(\lp(\State,\Action)\) defines a set of restrictions on an MDP. These restrictions are a set \(\forbS\) of states that the agent may not be able to visit and a set \(\forbA\) of actions that the agent may not be able to perform. Also, changing actions or states imply changing the transitions as well. Thus, \(\genS \ssq \State \: | \: \genS = \State - \forbS\) and \(\genA \ssq \Action \: | \: \genA = \Action - \forbA\).

The transition function \(\genT\) is then described considering the following conditions: 
\begin{enumerate}
    \item The agent cannot visit a state that is forbidden: \(\forbS \times \Action\);
    \item The agent cannot perform a forbidden action: \(\State \times \forbA\);
    \item The agent cannot perform a forbidden action in a state that it cannot visit: \(\forbS \times \forbA\);
    \item The agent cannot visit states that have no transition probabilities: \(\State \times \Action \times \State \mapsto 0\);
    \item The agent is not allowed to perform some specific actions in some specific states: \(\forbQ\).
\end{enumerate}

Thus, the transition function that is extracted from the answer sets is defined as: 
\begin{align}
\genT(\genS, \genA) =& (\genS \times \genA) - \left( (\forbS \times \Action) + 
                       (\State \times \forbA) + (\forbS \times\forbA) \right. 
                       \nonumber \\
                     & \left. + (\State \times \Action \times \State \mapsto 0) 
                       + (\forbQ) \right) \mapsto \genS
\end{align}

    When an MDP is not deterministic, choice rules are used to describe the transition possibilities (without the probability itself), a similar process is used to find the transition function \(\genT\).

Therefore, with this new set of states \(\genS \ssq \State\), actions \(\genA \ssq \Action\) and transition function \(\genT\), it is possible to formalise an MDP \(\genMDP \ssq \MDP\) in the form \(\genMDPdef\). Since the reward comes from the interaction with the environment, there is no need to suppress any value in this function or even to know which is the reward function beforehand.
\end{proof}

Once it is possible to formalize an MDP \(\genMDP\) that is a subset of another MDP \(\MDP\), it is still necessary to guarantee that the optimal solution \(\optpol_{\genMDP}\) of \(\genMDP\) is the optimal solution of \(\optpol_\MDP\) of \(\MDP\) as stated in Theorem~\ref{theorem:mdp}.

\begin{theorem}\label{theorem:mdp} 
    Given a reward function and an evaluation criteria (i.e. maximasing or minimising rewards), the optimal solution \(\optpol_{\genMDP}\) for the MDP \(\genMDP \ssq \MDP\) is equivalent to the optimal solution \(\optpol_\MDP\) for the MDP \(\MDP\) given the answer sets (trajectories) \(\Traj\) found as solutions to the logic program \(\lp(\State,\Action)\) that represents \(\MDP\).
\end{theorem}
\begin{proof}
Both \(\MDP\) and \(\genMDP\) have to maximise (or minimise) the same reward function \(\Reward\). If there is no restrictions in the set of states (\(\forbS = \Empty\)) and actions (\(\forbA = \Empty\)), we have that \(\genMDP = \MDP\) and \(\optpol_\MDP = \optpol_{\genMDP}\).

If there are restrictions represented in \(\lp(\State,\Action)\), then \(\genMDP \subset \MDP\) and the feasible solutions (answer sets) \(\Traj\) for \(\MDP\) are the same of those for \(\genMDP\) (by using Lemma~\ref{lemma:traj}). Since the optimal solution must be a feasible solution, then \(\optpol_\MDP \in \Traj\) and \(\optpol_{\genMDP} \in \Traj\). Thus, given the same set of feasible solutions and the same evaluation criteria, \(\optpol_\MDP = \optpol_{\genMDP}\).
\end{proof}

\subsection{The Algorithm \ASPRL{}}\label{sec:3.2}

Lemma~\ref{lemma:traj} and Theorem~\ref{theorem:mdp} support the use of ASP to find the sets of states and actions of an MDP. By using RL it is possible to find an optimal stochastic solution to this MDP. Since ASP allows for revisions to be made in the set of states and actions, if it is the case that the environment changes at any time step, it can be used to find the new subsets \(\genS\) of states and \(\genA\) of actions of the modified MDP and values learnt from the previous interaction can be used as input for this new MDP. Algorithm~\ref{alg:asprl} is the pseudocode of \ASPRL{}, that uses the non-monotonicity of ASP along with the exploratory nature of RL algorithms in stochastic domains. 
\begin{algorithm}[t]
\TitleOfAlgo{\ASPRL{}}
\KwIn{An MDP descried as a logic program \(\lp(\State,\Action)\) and a
(optional) \(\vaf(\state,\action)\) function to be approximated.}
\KwOut{The approximated \(\vaf(\state,\action)\) function.}
\DontPrintSemicolon
Find the answer sets \(\Traj\) for \(\lp(\State,\Action)\).\;
Update \(\vaf(\state,\action)\) function using \(\genS\) and \(\genA\) found in \(\Traj\).\;
\While{the environment does not change}{Approximate \(\vaf(\state,\action)\) using a RL method.\;}
%\KwRet{the \(\vaf(\state,\action)\) function approximated.}
Include the observed changes in \(\lp(\State,\Action)\) \;
Call \ASPRL{} with \(\lp(\State,\Action)\) and the \(\vaf(\state,\action)\) function approximated.
\caption{\ASPRL{} Algorithm.}
\label{alg:asprl}
\end{algorithm}

Algorithm~\ref{alg:asprl} uses RL methods for approximating the \(\vaf(\state,\action)\) function for the states and actions obtained by ASP. First, the domain is described as a logic program \(\lp(\State,\Action)\), using the \bcp{} vocabulary, and answer sets are found for it. From those answer sets (as shown in Lemma~\ref{lemma:traj}) the sets of states \(\genS\) and actions \(\genA\) are constructed for the MDP that will be used by the agent to interact with the environment, along with the transition function \(\genT\). Once the MDP \(\genMDPdef\) is formalised, the interaction with the environment and the search for the optimal solution begins by using any RL algorithm. This interaction continues until a change in the environment happens. At this instant, the algorithm returns the approximated \(\vaf(\state,\action)\).

The algorithm works in non-stationary environments by including the observed environment changes in \(\lp(\State,\Action)\) so that ASP can be used again to find the new sets of states and actions along with the transition function. Since there is a \(\vaf(\state,\action)\) function approximated from the previous interaction, modifications are performed in it. The state-action pairs that are in the new set of answer sets are added to the action-value function and the pairs that are not in this set are removed. The state-action pairs that were in the function, and that are also in the answer set, remain in the action-value function with the previously learned value. Therefore, the interaction with a changing environment is done by calling \ASPRL{} with \(\lp(\State,\Action)\), augmented with the observed changes, and the action-value function returned by the previous call.

\section{Experiments}\label{sec:exp}

Experiments were performed in a non-deterministic non-stationary grid world of size $10\times 10$ which allowed the execution of only one of four actions each time: \gup{}, \gdown{}, \gleft{} and \gright{}. The probabilities for the environment were defined as 80\% for the  transition to happen as expected (e.g., executing \gup{} makes the agent go up with 80\% of probability) and a 20\% chance for the agent to go orthogonal to the desired direction (e.g., executing \gup{} may make the agent to go left or right with 10\% of chance for each side).

The grid world may have walls (W) and holes (H) each of which occupies a single cell of the grid. When the agent performs an action and hits a wall, it stays in the same state; when it executes an action and falls into a hole, the episode ends. In this domain, an agent that starts in the lowermost, leftmost, cell has as a goal to reach the topmost, rightmost, cell. The reward function used in this domain is \(+100\) for reaching the goal, \(-100\) for falling in a hole and \(-1\) in any other event. It is important to notice that the transition function and reward function are unknown to the agent. For this grid world, the representation used by the agent is the value of its position in X and Y. These values are not treated by the agent as an X by Y matrix, but as a set of atoms in the form \((X,Y)\) for each pair of X and Y values found in an answer set.

\begin{figure*}[t!]
\centering
\begin{subfigure}[t]{0.45\textwidth}
  \centering
  \includegraphics[width=\textwidth]{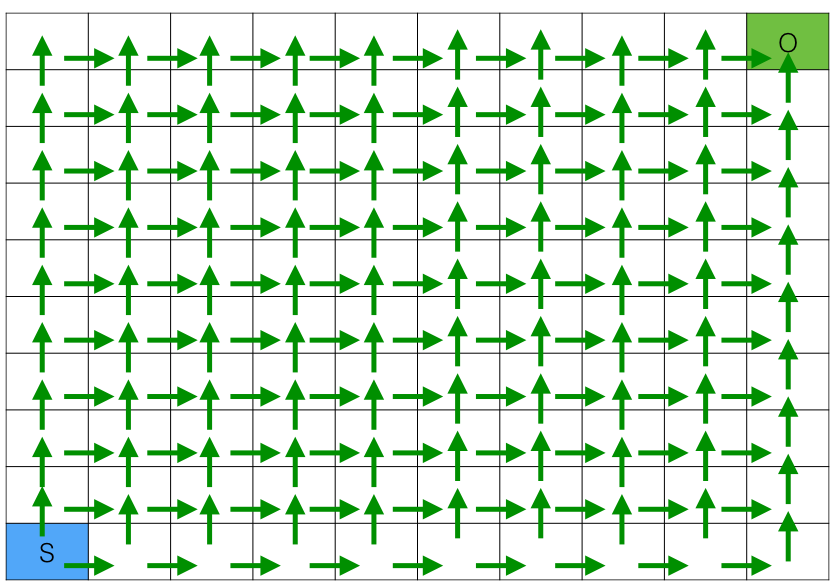}
  \caption{Map 1: Initial configuration.}
  \label{fig:map1}
\end{subfigure}
\hfill
\begin{subfigure}[t]{0.45\textwidth}
  \centering
  \includegraphics[width=\textwidth]{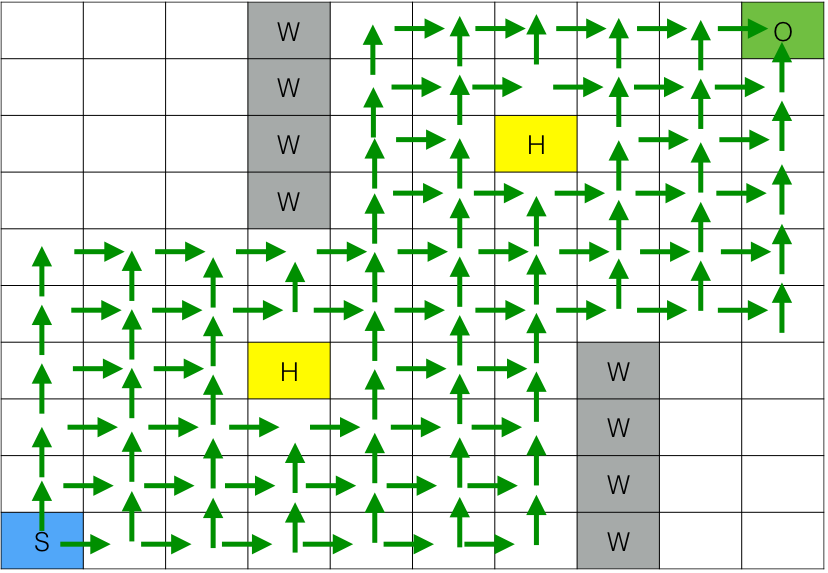}
  \caption{Map 2.}
  \label{fig:map2}
\end{subfigure}
\\
\begin{subfigure}[t]{0.45\textwidth}
  \centering
  \includegraphics[width=\textwidth]{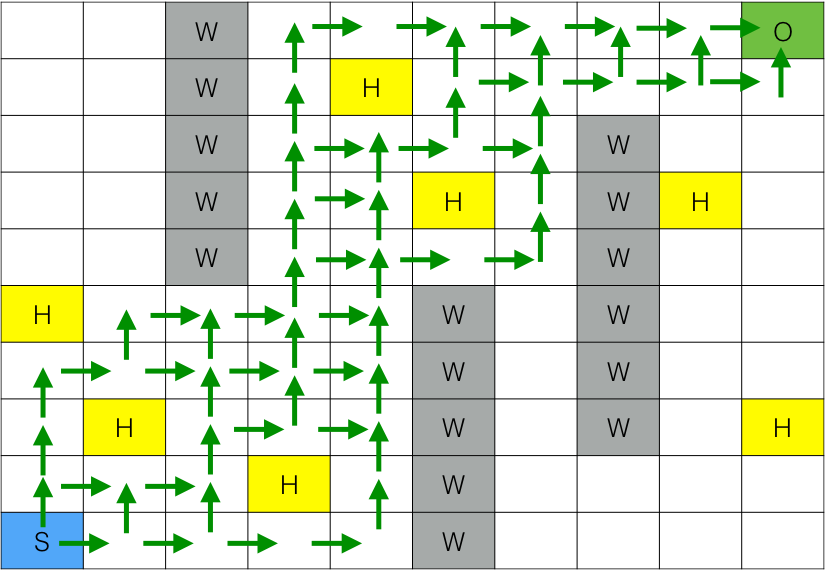}
  \caption{Map 3.}
  \label{fig:map3}
\end{subfigure}
\hfill
\begin{subfigure}[t]{0.45\textwidth}
  \centering
  \includegraphics[width=\textwidth]{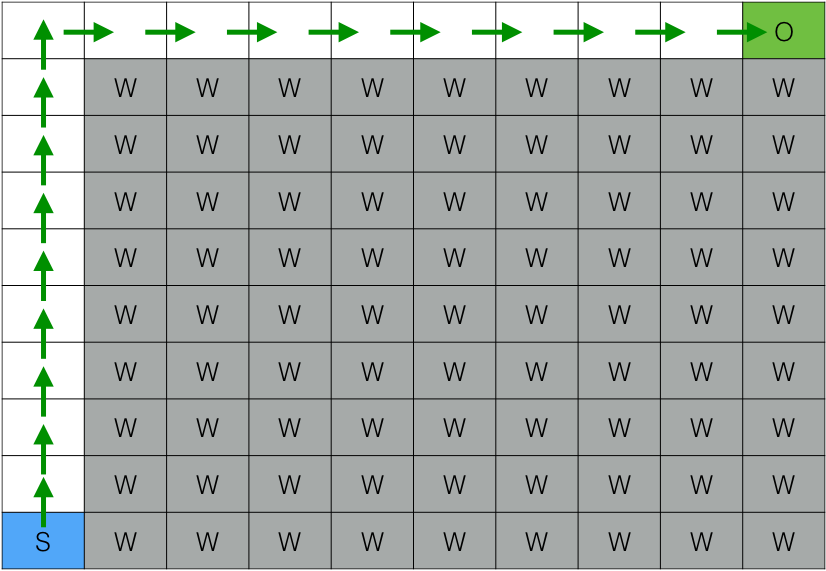}
  \caption{Map 4.}
  \label{fig:map4}
\end{subfigure}
\caption{Grid worlds used in the experiments. Squares labelled with W
  represent walls and with H, holes.}
\label{fig:gridworld}
\end{figure*}

This grid world suffers changes in a manner that is previously unknown to the learning agent. In this work, \ASPRL{} is evaluated in three distinct situations, in each of them the agent starts in the map shown in Figure~\ref{fig:map1} that, after 5000 episodes, changes to one of the other maps in Figure~\ref{fig:gridworld}. For this work, changes observed in the environment were manually entered in the logic program. Nevertheless, this can be automatically done by using an online method with ASP.

The map in Figure~\ref{fig:map1} represents a grid world with no walls or holes. In this case, any combination of actions that makes the agent to go {\em up} and {\em right} leads the agent to the goal. Figure~\ref{fig:map2} represents a grid world with two walls and two holes. Figure~\ref{fig:map3} shows a grid world containing more walls and holes than in the previous situation. In this case, the agent has fewer action options to achieve the goal state. Finally, Figure~\ref{fig:map4} represents a grid world in which there is only one policy for achieving the goal state with the minimum number of actions. Any other policy for this grid world will necessarily make the agent hit a wall before reaching the goal state.

The arrows in the maps shown in Figure~\ref{fig:gridworld} represent the feasible policies obtained by ASP with the minimum number of steps. Note that, these policies do not represent the transition probabilities of the environment.

In the first situation, the environment changes from the map in Figure~\ref{fig:map1} to that in Figure~\ref{fig:map2}. In this case, we can see that there is a reduction in the number of policies with the minimum number of steps.

In the second situation, the change occurs from the map shown in Figure~\ref{fig:map1} to that in Figure~\ref{fig:map3}. By analysing the arrows in the final (Fig.~\ref{fig:map3}) grid world, we can see that there is an even greater reduction in the number of policies than in the previous situation, since there are more walls and holes in this map, which imply fewer safe actions ({\em arrows}) available.

In the final situation the environment changes from the map in Figure~\ref{fig:map1} to the one in Figure~\ref{fig:map4}. In this case the MDP has only one optimal solution. This situation was chosen since the answer set provides the only optimal solution almost instantly, whereas in the case where an action-value function is approximated by RL (without using the answer sets), every possible action in every possible state is considered, leading to a costly search procedure. 

\section{Results}\label{sec:res}

In this section we use the situations described above to compare the learning processes of SARSA and Q-Learning with those of \ASPSARSA{} and \ASPQ{}, which are \ASPRL{} methods where SARSA and Q-Learning are used along with ASP. This comparison is accomplished with two different criteria: the return (\(\sum \reward(\state, \action, \fstate)\)) of the episode and the number of steps needed to reach the goal state and root-mean-square deviation (RMSD) of the action-value function, at time \(t\) wrt time \(t-1\), according to the equation \ref{eq5} below.
\begin{align}\label{eq5}
  \mathrm{RMSD} = \sqrt{\frac{\sum^n_{m=1}{\vaf{}_{m,\timeinst}(\state,\action) - \vaf{}_{m,\timeinst-1}(\state,\action)}}{n}}
\end{align}

\begin{figure*}[p!]
\centering
\begin{subfigure}[t]{0.85\textwidth}
	\centering
	\includegraphics[width=\columnwidth]{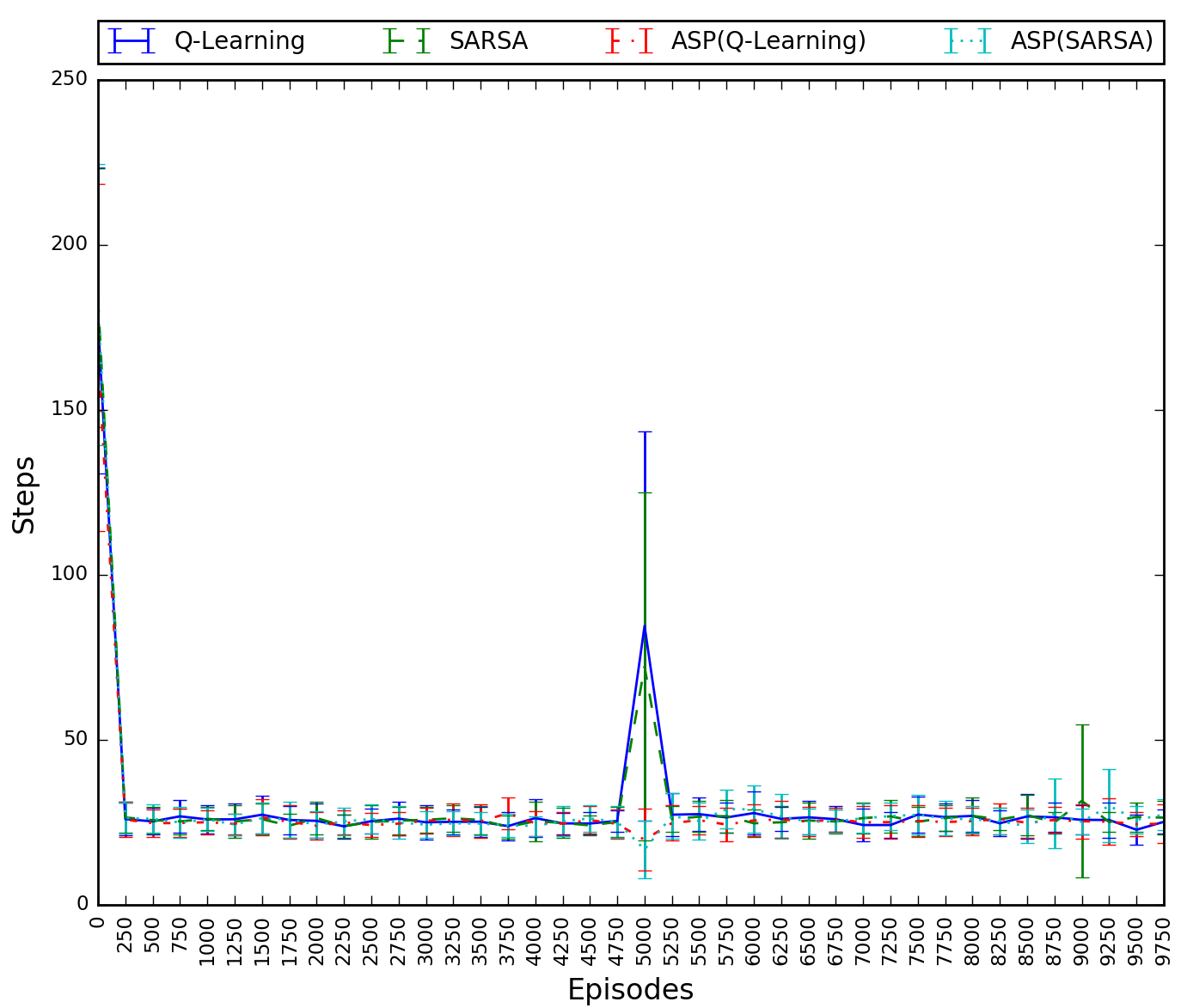}
	\caption{Steps needed to reach the goal state.}
	\label{fig:steps12c}
\end{subfigure}
\\
\begin{subfigure}[t]{0.85\textwidth}
	\centering
	\includegraphics[width=\columnwidth]{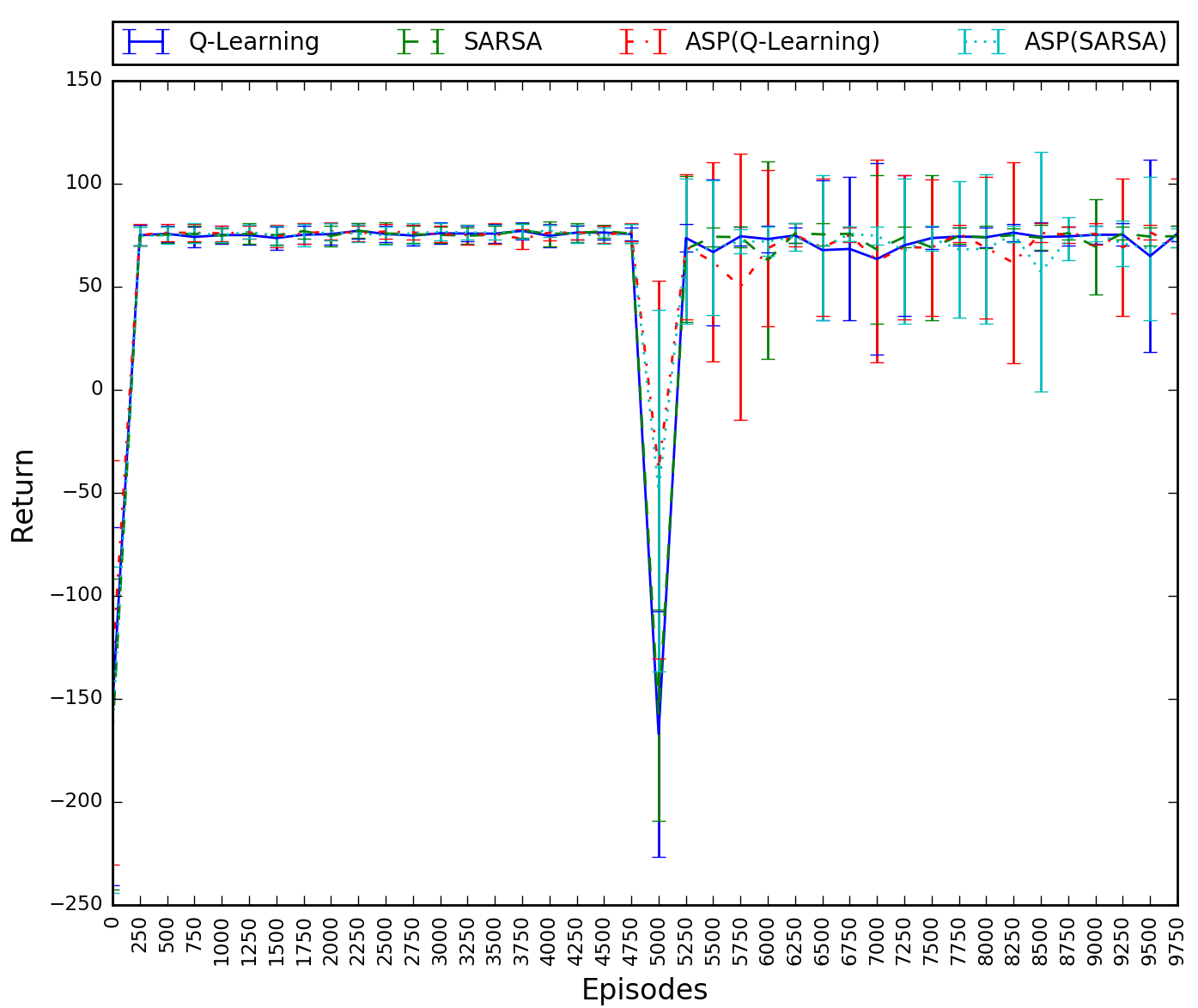}
	\caption{Total returns receive per episode.}
	\label{fig:ret12c}
\end{subfigure}
\caption{Results for the first situation for every episode.}
\label{fig:resexp1c}
\end{figure*}

\begin{figure*}[p!]
\centering
\begin{subfigure}[t]{0.85\textwidth}
	\centering
	\includegraphics[width=\columnwidth]{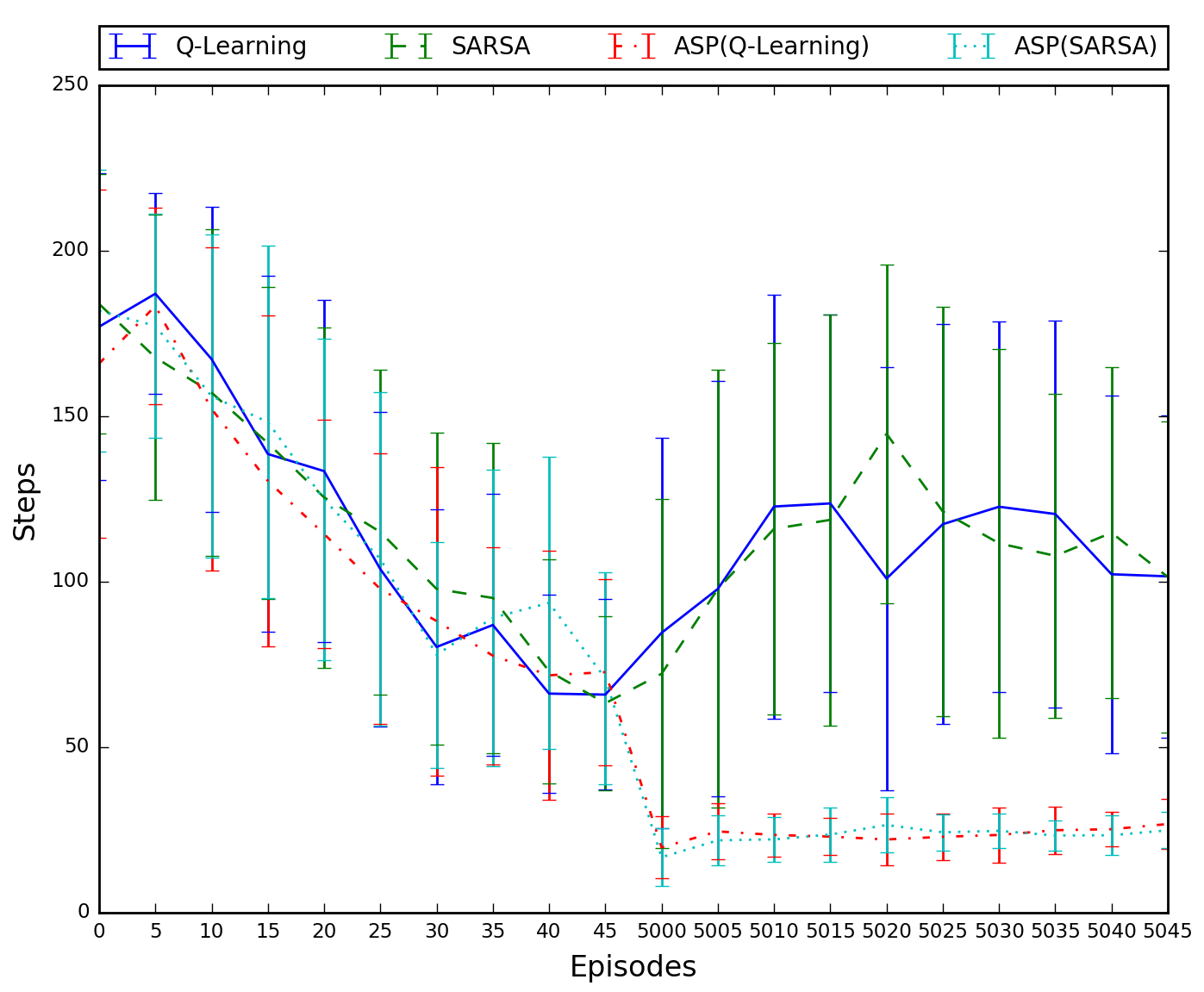}
	\caption{Steps needed to reach the goal state.}
	\label{fig:steps12}
\end{subfigure}
\\
\begin{subfigure}[t]{0.85\textwidth}
	\centering
	\includegraphics[width=\columnwidth]{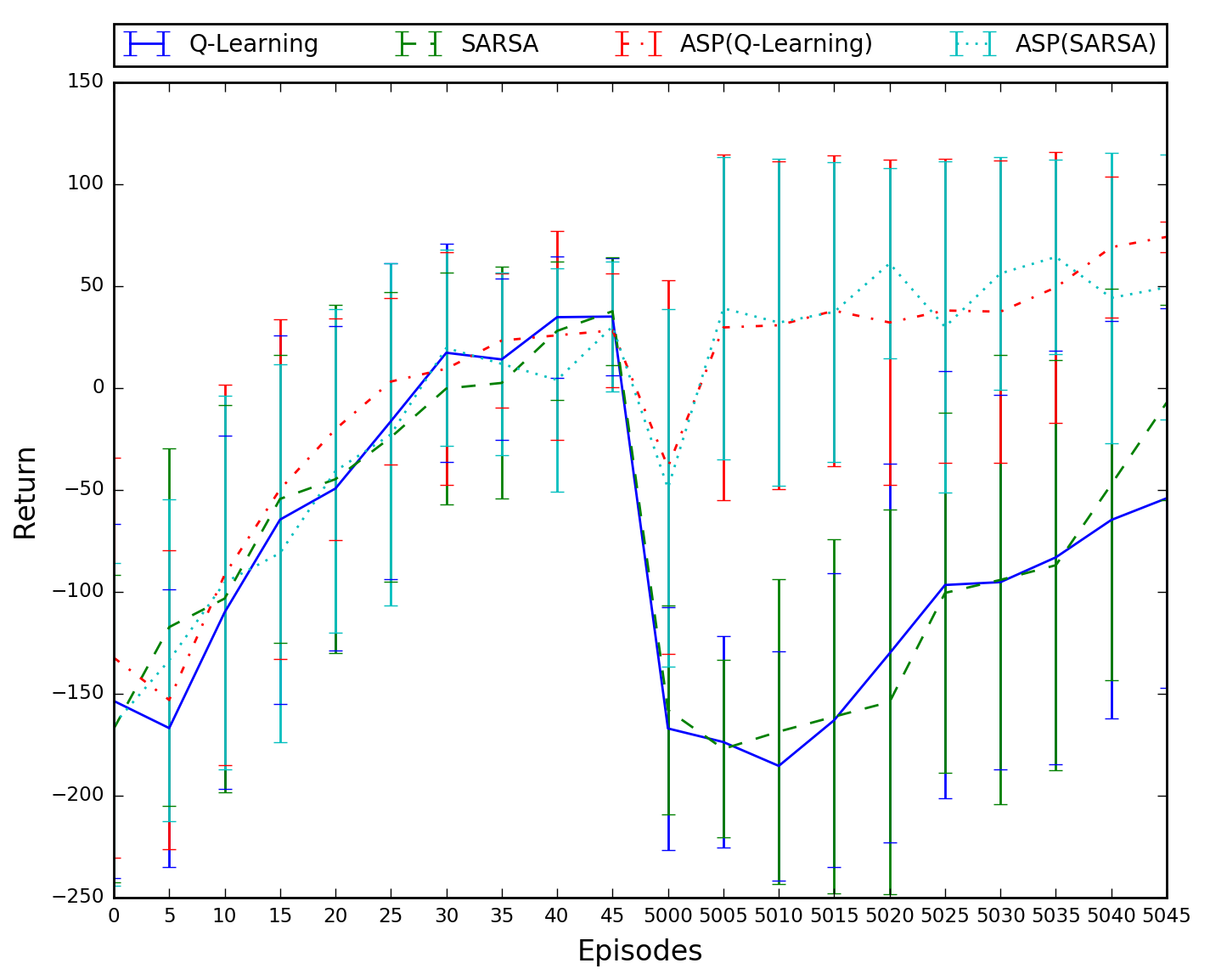}
	\caption{Total returns receive per episode.}
	\label{fig:ret12}
\end{subfigure}
\caption{Results for the first situation.}
\label{fig:resexp1}
\end{figure*}

\begin{figure*}[p!]
\centering
\begin{subfigure}[t]{0.85\textwidth}
	\centering
	\includegraphics[width=\columnwidth]{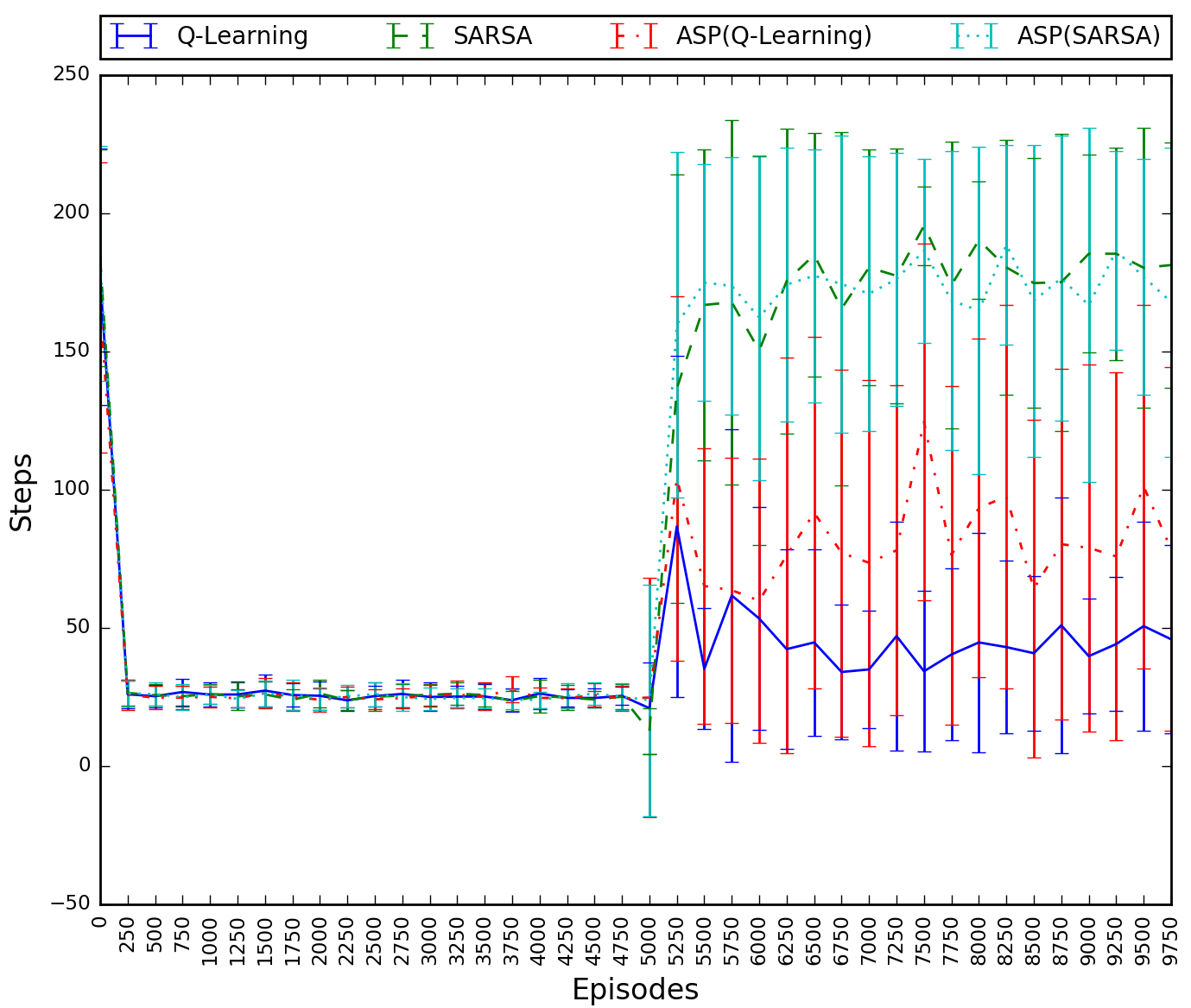}
	\caption{Steps needed to reach the goal state.}
	\label{fig:steps13c}
\end{subfigure}
\\
\begin{subfigure}[t]{0.85\textwidth}
	\centering
	\includegraphics[width=\columnwidth]{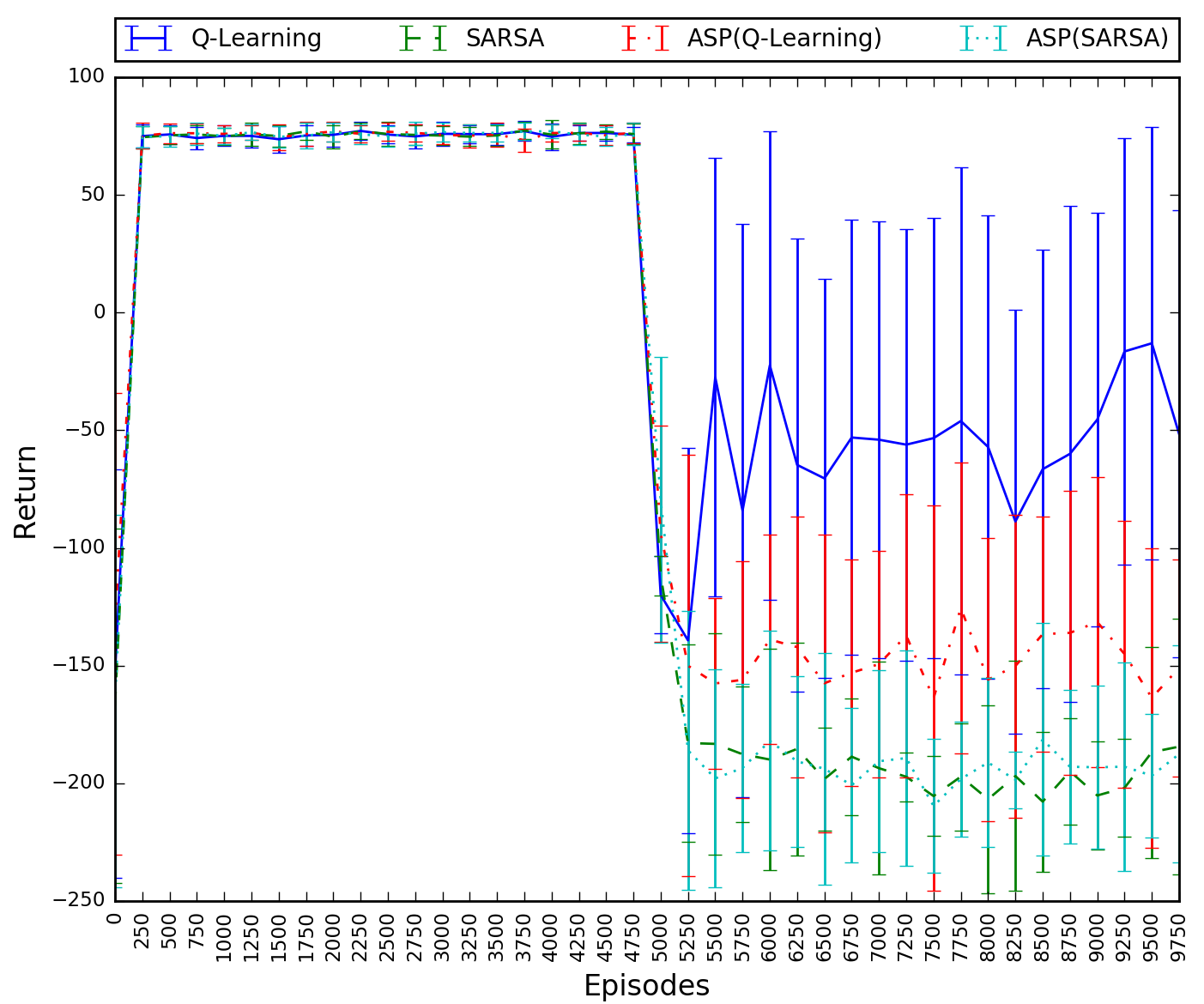}
	\caption{Total returns receive per episode.}
	\label{fig:ret13c}
\end{subfigure}
\caption{Results for the second situation for every episode.}
\label{fig:resexp2c}
\end{figure*}

\begin{figure*}[p!]
\centering
\begin{subfigure}[t]{0.85\textwidth}
	\centering
	\includegraphics[width=\columnwidth]{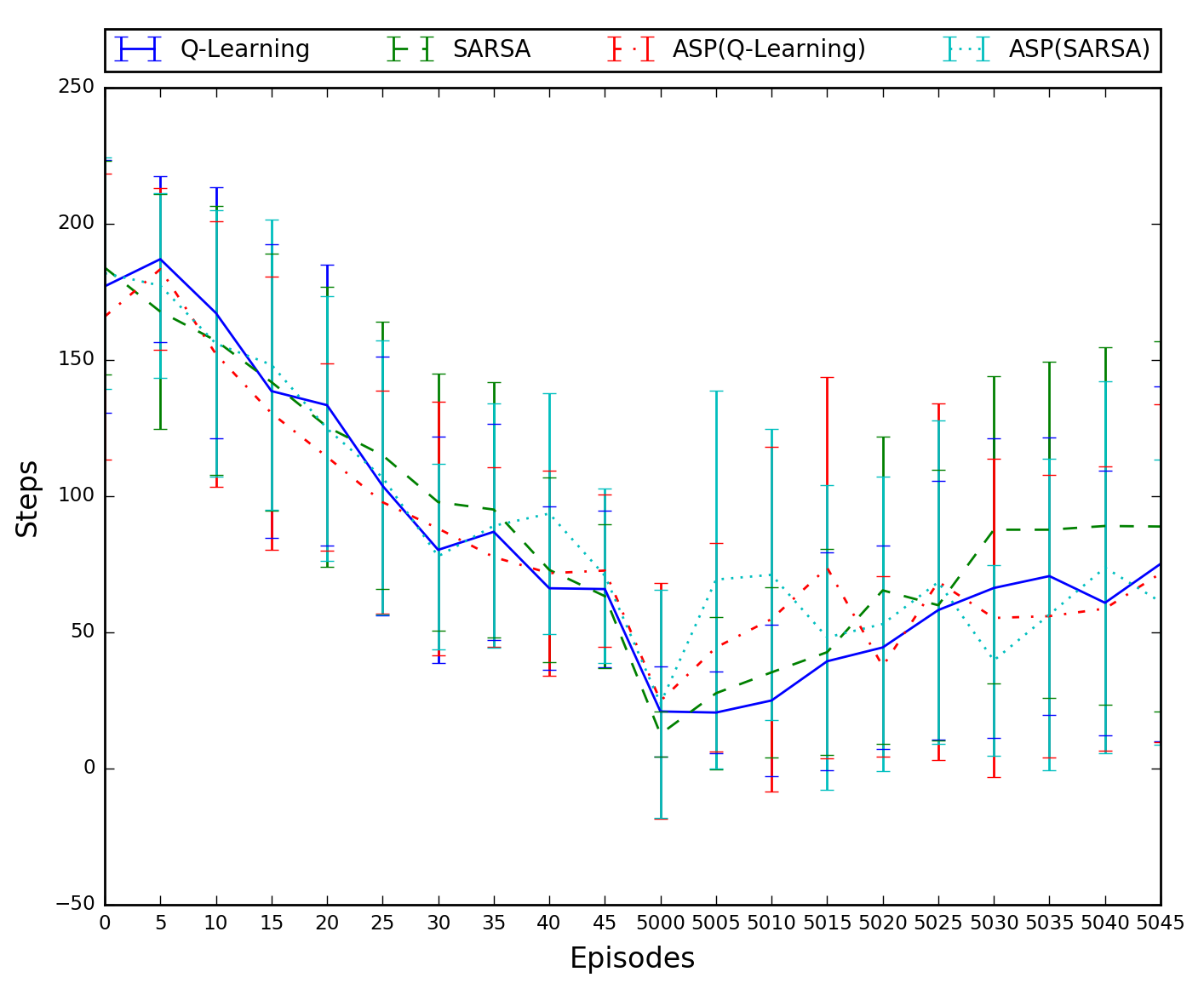}
	\caption{Steps needed to reach the goal state.}
	\label{fig:steps13}
\end{subfigure}
\\
\begin{subfigure}[t]{0.85\textwidth}
	\centering
	\includegraphics[width=\columnwidth]{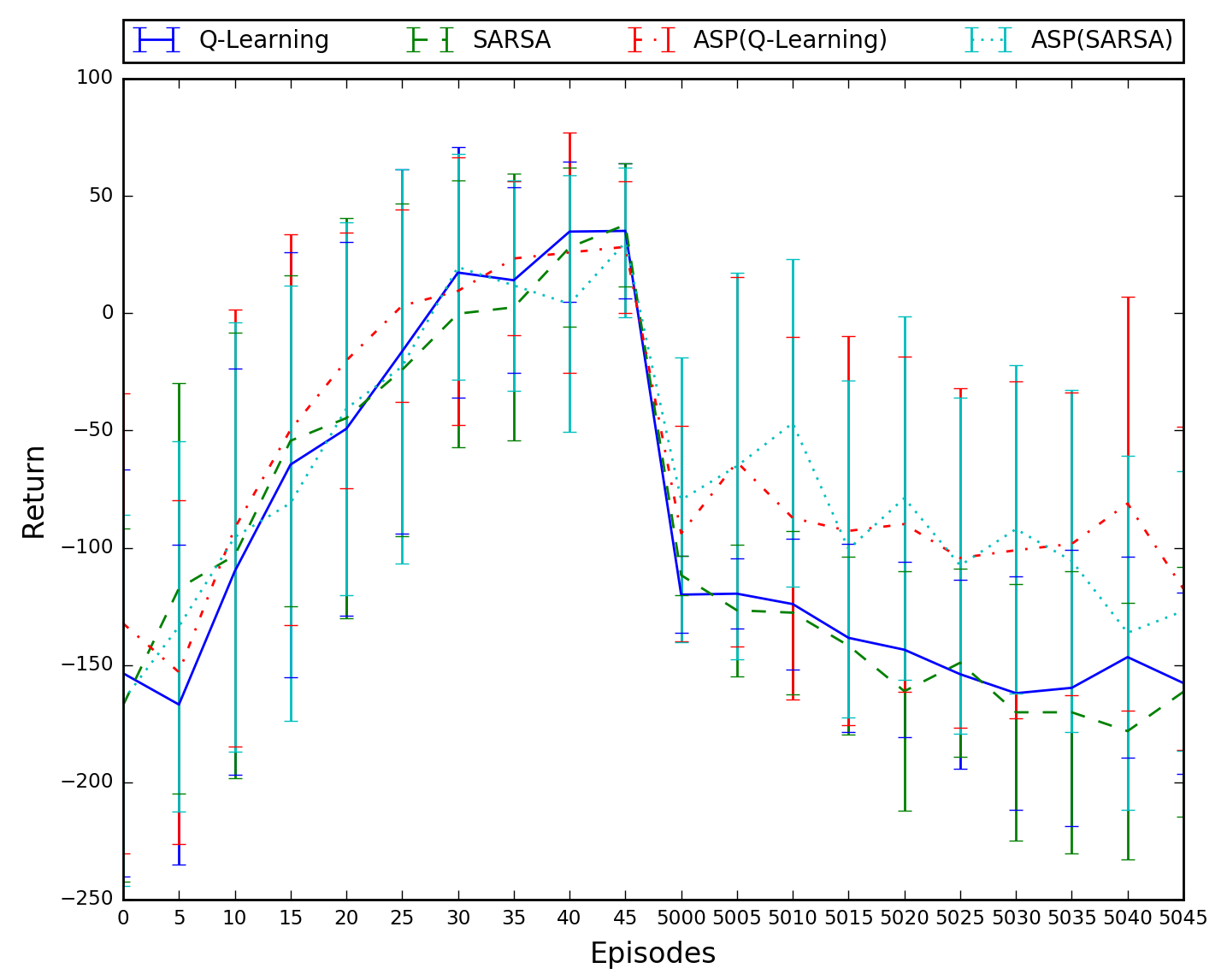}
	\caption{Total returns receive per episode.}
	\label{fig:ret13}
\end{subfigure}
\caption{Results for the second situation.}
\label{fig:resexp2}
\end{figure*}

\begin{figure*}[p!]
\centering
\begin{subfigure}[t]{0.85\textwidth}
	\centering
	\includegraphics[width=\columnwidth]{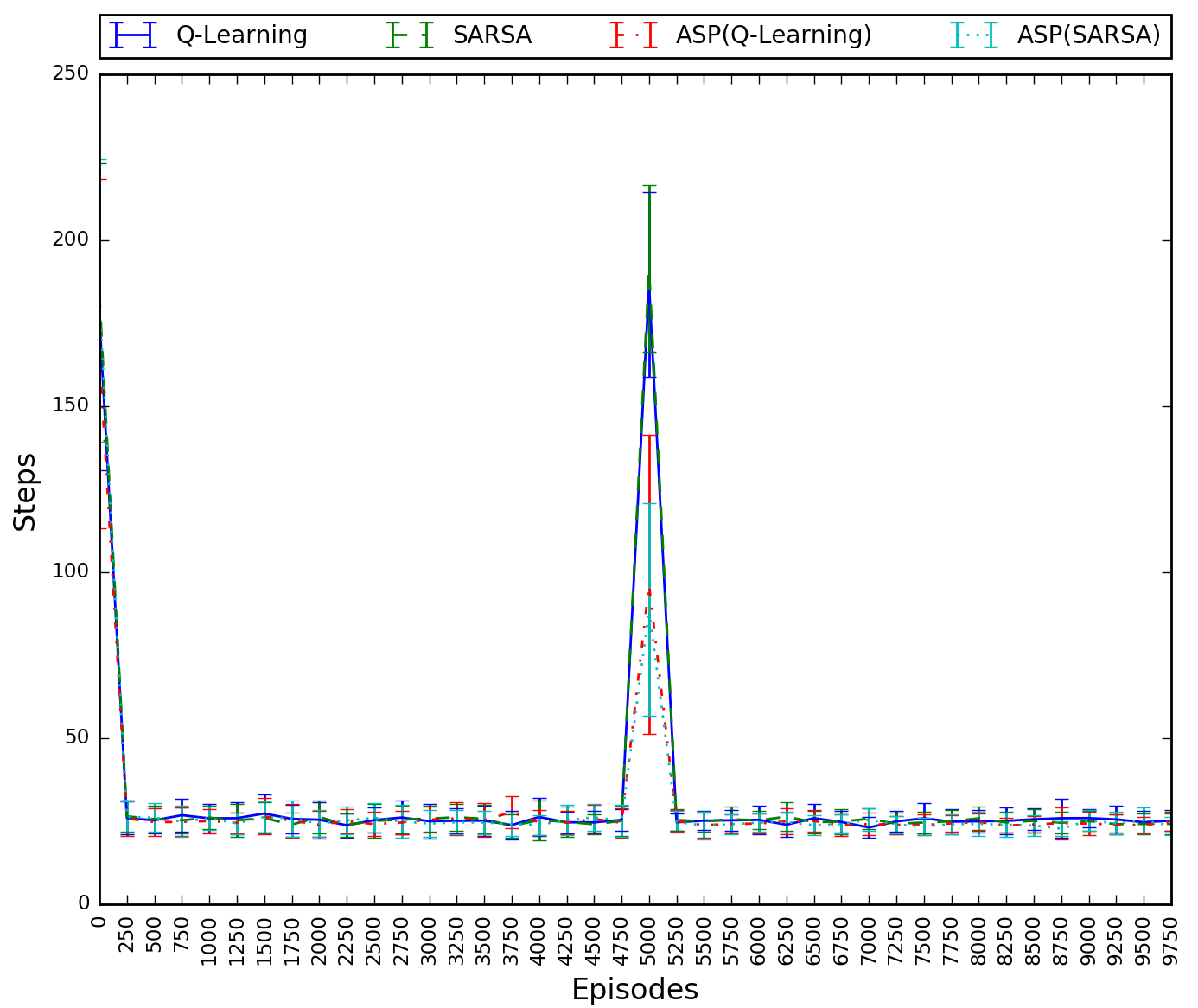}
	\caption{Steps needed to reach the goal state.}
	\label{fig:steps14c}
\end{subfigure}
\\
\begin{subfigure}[t]{0.85\textwidth}
	\centering
	\includegraphics[width=\columnwidth]{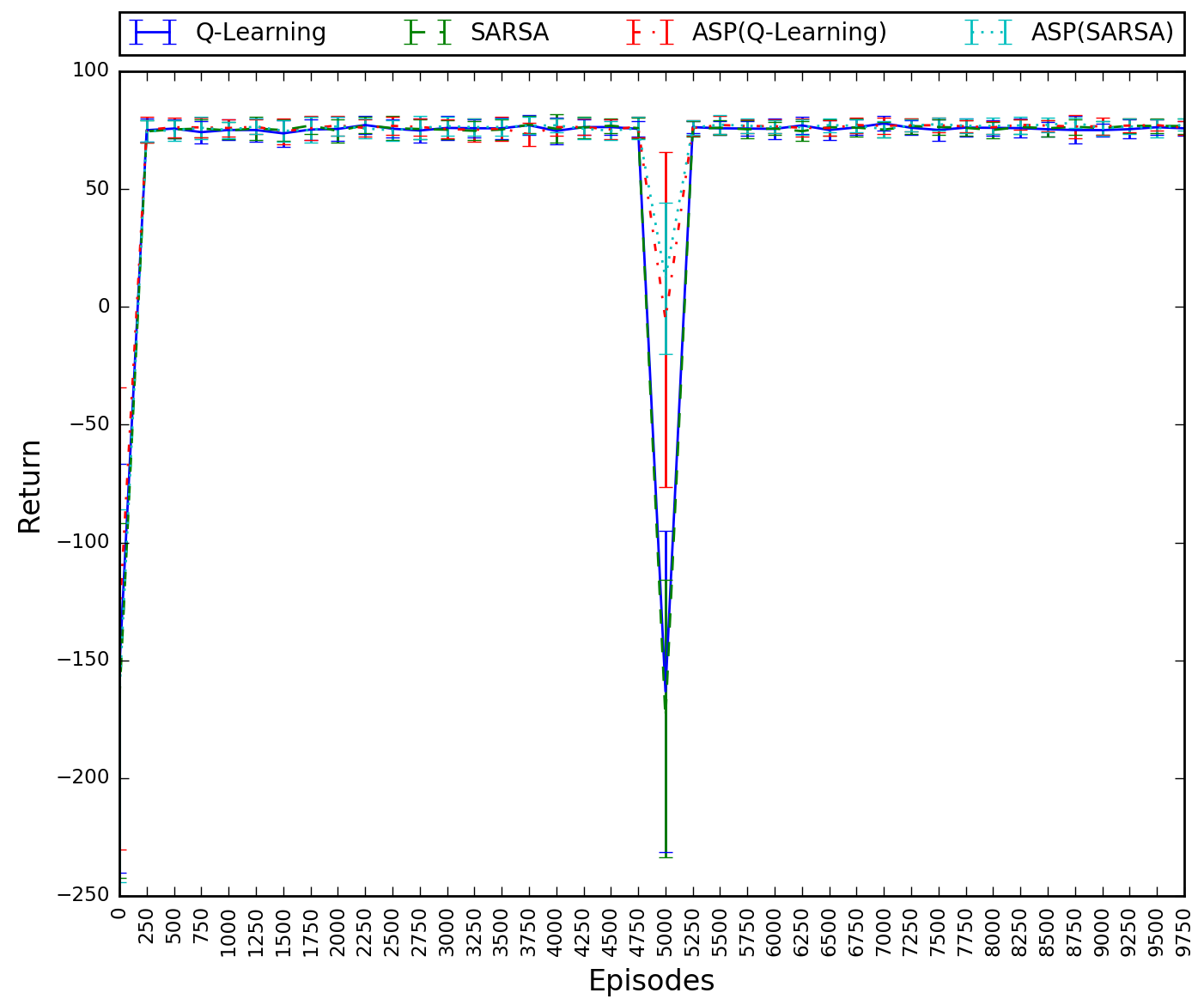}
	\caption{Total returns receive per episode.}
	\label{fig:ret14c}
\end{subfigure}
\caption{Results for the third situation for every episode.}
\label{fig:resexp3c}
\end{figure*}

\begin{figure*}[p!]
\centering
\begin{subfigure}[t]{0.85\textwidth}
	\centering
	\includegraphics[width=\columnwidth]{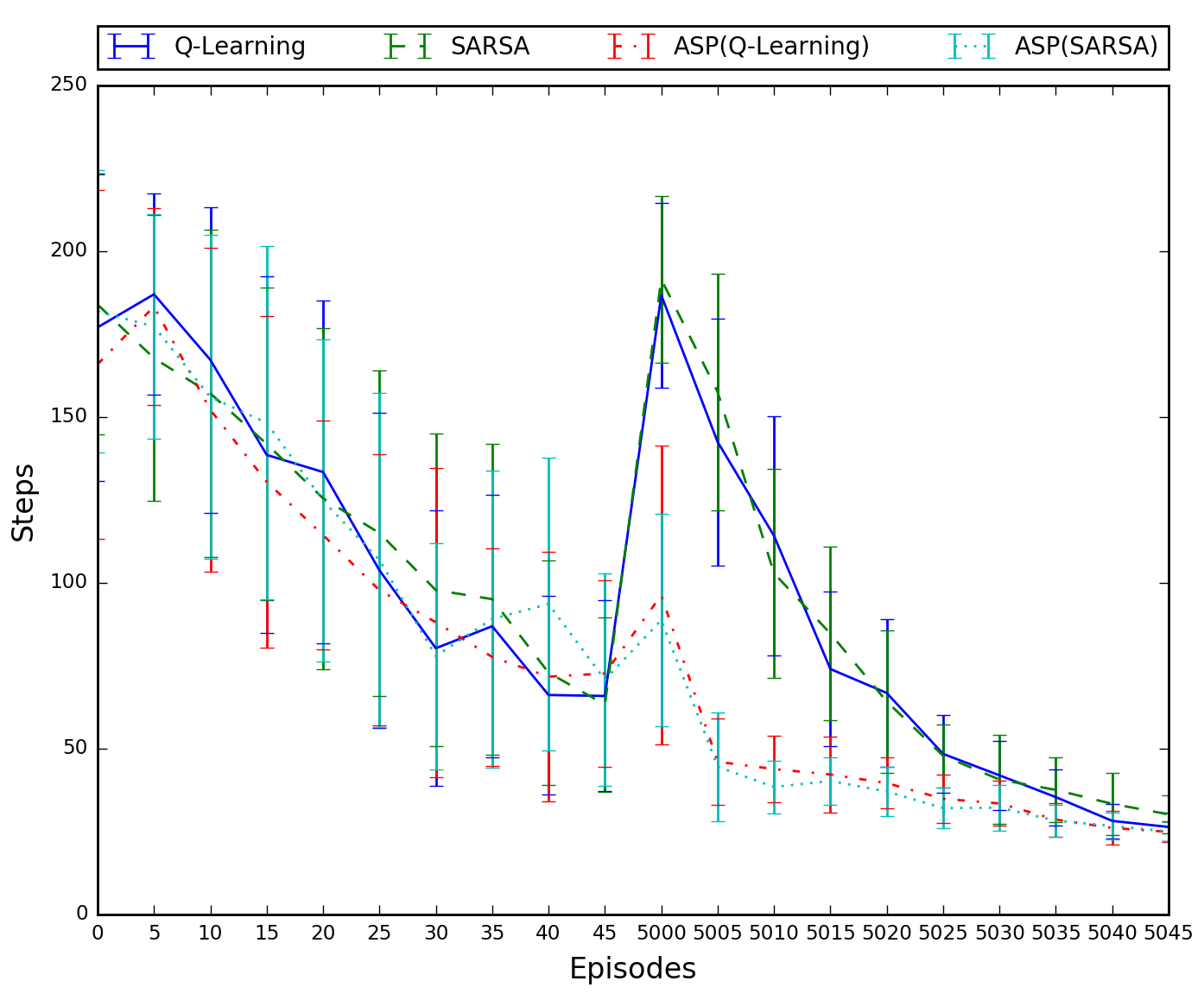}
	\caption{Steps needed to reach the goal state.}
	\label{fig:steps14}
\end{subfigure}
\\
\begin{subfigure}[t]{0.85\textwidth}
	\centering
	\includegraphics[width=\columnwidth]{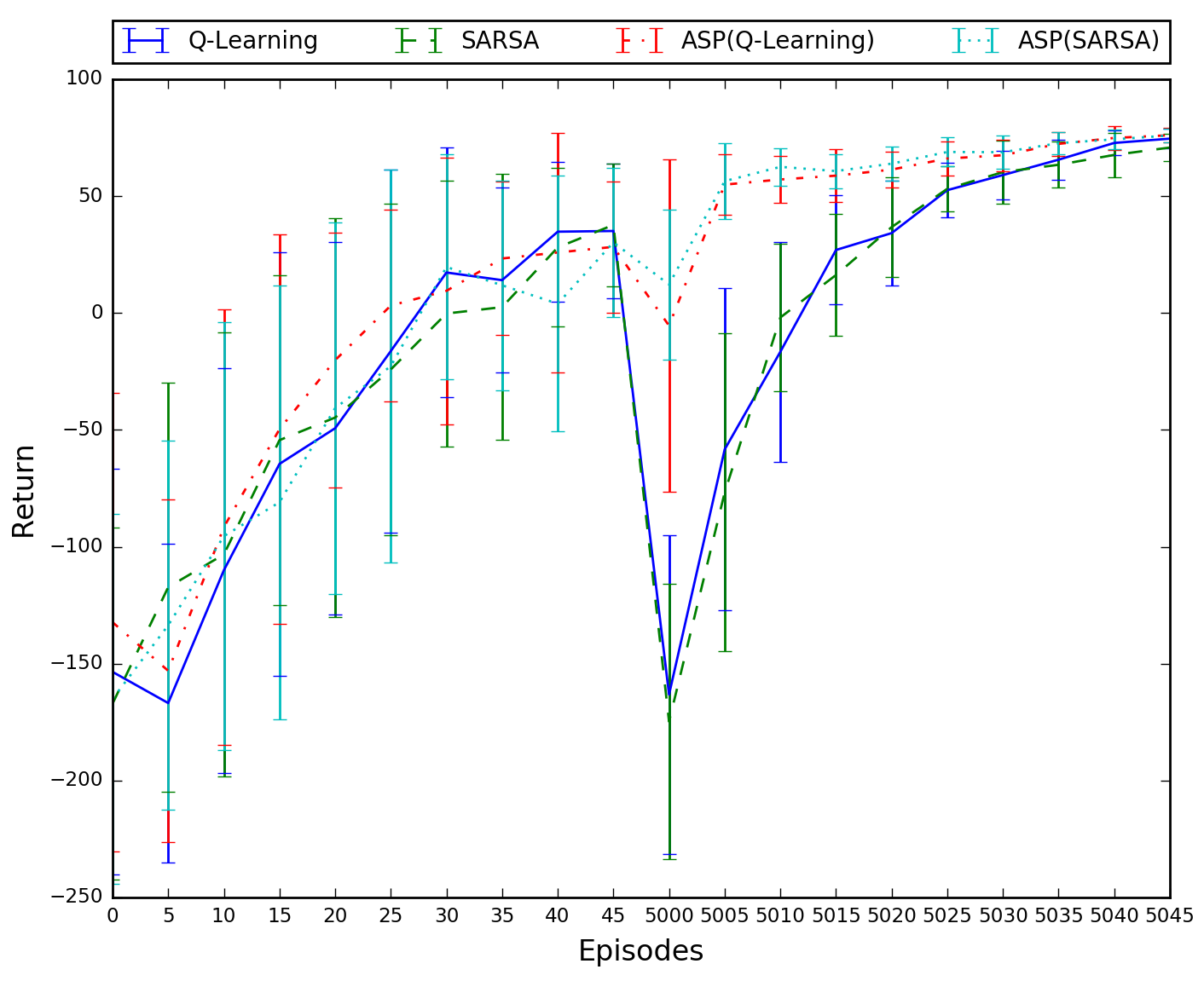}
	\caption{Total returns receive per episode.}
	\label{fig:ret14}
\end{subfigure}
\caption{Results for the third situation.}
\label{fig:resexp3}
\end{figure*}

The graphs in Figures~\ref{fig:resexp1c}, \ref{fig:resexp1}, \ref{fig:resexp2c}, \ref{fig:resexp2}, \ref{fig:resexp3c} and~\ref{fig:resexp3} represent measurements for the four algorithms applied in the three situations considered. Figures~\ref{fig:resexp1}, \ref{fig:resexp2} and~\ref{fig:resexp3} depict the results of the first 50 episodes for the first map and then skipping to the 5000th episode directly, in order to present the measurements after the environment change occurs\footnote{Episodes 51st to 4999th were removed from the figures.}. The respective results for every episode are shown in Figures~\ref{fig:resexp1c}, \ref{fig:resexp2c} and~\ref{fig:resexp3c}.

The results for the number of steps from the first situation are presented in Figure~\ref{fig:steps12}. Figure~\ref{fig:ret12} presents the returns. In the first map (Figure~\ref{fig:map1}) all four algorithms present the same number of steps and returns during the initial 50 episodes shown. After the 5000th episode, the number of steps of \ASPQ{} and \ASPSARSA{} decrease faster than that of Q-Learning and SARSA, while the returns of \ASPQ{} and \ASPSARSA{} increase faster than Q-Learning and SARSA. This difference in the performance of \ASPRL{} and RL algorithms after the change occurs in the environment is due to the fact that \ASPRL{} reuses the \(\vaf(\state,\action)\) function approximated in the previous map.

For the second situation, Figure~\ref{fig:steps13} presents the number of steps and Figure~\ref{fig:ret13} the returns for the four algorithms. Regarding the number of steps, it is possible to notice that although \ASPQ{} and \ASPSARSA{} use the information acquired from previous experience, they still need the same number of steps as Q-Learning and SARSA in all episodes. However, the returns for \ASPQ{} and \ASPSARSA{} are higher than the returns from Q-Learning and SARSA when the change in the map occurs (5000th episode). This similarity in the number of steps for the four algorithms is due to the great change that occurred in the environment, thus \ASPQ{} and \ASPSARSA{} still need to learn interactively with the new environment, even though they use information from the previous map.

The number of steps and returns for the third situation are presented in Figure~\ref{fig:steps14} and~\ref{fig:ret14} respectively. In both returns and steps, when the change occurs, the use of previously learned values enhance the performance of \ASPQ{} and \ASPSARSA{}. While there is a slow decrease in the number of steps in Q-Learning and SARSA and a slow increase in the returns, \ASPQ{} and \ASPSARSA{} can quickly learn the only optimal policy since this policy is already known from previous map.

Experiments were performed in a 1.66GHz Core2Duo with 4GB of RAM running Debian 9 (currently the testing version). Logic programs were written in \(BC+\)~\cite{BCplus2015} and translated to ASP language using \cpasp{}~\cite{babb2013}, which uses \iclingo{}~\cite{ASP2012} to find answer sets. For finding the optimal solution, Q-Learning and SARSA were implemented in Python 3.5 using only built-in libraries. Thirty training sessions were executed for each  algorithm. The same parameters were used in all the experiments: learning rate $\alpha = 0.2$, discount factor $\gamma = 0.9$, exploration/ exploitation rate $ \epsilon = 0.1$ and the Q table was randomly initialised.

\section{Discussion}

The results shown in the previous section, present the best, worst and average cases of the \ASPRL{} method proposed.

The first map (Figure~\ref{fig:map1}) represents the worst case for \ASPRL{}. As can be seen in the graphs in Figures~\ref{fig:resexp1c}, \ref{fig:resexp1}, \ref{fig:resexp2c}, \ref{fig:resexp2}, \ref{fig:resexp3c},\ref{fig:resexp3}, the performance of \ASPQ{} and \ASPSARSA{} are the same as that of Q-Learning and SARSA. This is due to the fact that the reduction in the sets of states and actions are minimal (since there isn't any restriction in this map) and \ASPRL{} methods use the same \(\State\) and \(\Action\) as an RL method.

The best case is represented in the last map (Figure~\ref{fig:map4}). In this case, there is only one feasible policy and, thus, this is the optimal policy. Although the learning process has been executed, in situations like these learning is not necessary, since there is only one feasible policy that is provided by an answer set. 

A similar case occurs when there is no feasible policy. In this situations there is also no need to perform the learning process, since it is already known from the answer sets that there is no feasible/optimal policy and the problem cannot be solved.

The average case is presented in the second and third maps (Figures~\ref{fig:map2} and~\ref{fig:map3} respectively). In these situation it is possible to notice that there is a reduction in the sets of states and actions, along with a reduction in the search space. Nevertheless, the acceleration in the learning process depends on how much the environment has changed from the previous situation. For example, the gain in learning time in the second situation (Figures~\ref{fig:resexp2c} and~\ref{fig:resexp2}) is greater than that of the third situation (Figures~\ref{fig:resexp3c} and~\ref{fig:resexp3}).

\ASPRL{} was not only capable of dealing with non-stationary non-deterministic environments but it also provides the possibility to reduce the search space, thus finding the optimal solution in fewer interactions with the environment than using RL alone. This reduction in the search space is related to the problem that is being solved and not only to the method proposed.

\section{Related Work}\label{sec:rel}

The method proposed in this paper is in line with the work reported in~\cite{Mohan2015IEEE,Mohan2015WHR} where ASP is used to find a description of the domain and RL is applied in the search for the optimal solution. Although both proposals combine similar tools, their use differ. While the present work formalises an MDP from the answer sets, the method proposed in~\cite{Mohan2015IEEE,Mohan2015WHR} finds only one answer set for the problem, where each atom in this set defines a hierarchical POMDP that has to be solved.

A related approach is the combination of ASP with action costs~\cite{khandelwal2014,yang2014}. Although this method also uses a logic program to describe the domain, it uses a method different from RL to find the action costs. At each action that is executed, the agent finds new plans to reach the goal; the update of the state-action pair's value is not based on the Temporal Difference method of RL.

Another work that also deals with sequential decision making is P-Log~\cite{p-log,Baral-PLog} which calculates transition probabilities from sampling the environment, but without considering the cost  of performing an action. The present work differs from P-Log in that our goal is to find the optimal solution regarding not only the transition probabilities, but also the action costs.

Also related to our work is Saturated Path-Constrained MDP (SPC-MDP)~\cite{Kolobov2014}. In a SPC-MDP, a solution is found by a constraint satisfaction procedure. This closely relates to the results obtained with the use of ASP to define the set of states for an MDP as proposed in this paper. However, while the approach described in~\cite{Kolobov2014} uses a Dynamic Programming algorithm to find the solutions, \ASPRL{} uses the interaction with the environment in order to approximate the action-value function in non-stationary decision making problems, which (to the best of our knowledge) has never been attempted before.

Works that are somewhat related to our approach, but can be used when searching for the optimal policy are the ones that deals with changing reward functions, such as \cite{Experts,oMDP,ArbRew}. Since \ASPRL{} uses RL, changes in the reward function are learned by the agent and does not affect the algorithm. Another approach that is somewhat related to \ASPRL{} is hierarchical MDPs (such as the works of \cite{RL-TOPS,AbsBeh}), which can also be incorporated such as the method proposed by \cite{Mohan2015IEEE,Mohan2015WHR} described in the beginning of this section. Although the decomposition proposed by hierarchical MDPs provide more abstraction when searching for the solution, \ASPRL{} deals with changes in \(\State\) and \(\Action\) such as the number of states and actions available in the environment or their representation.

To the best of our knowledge, these are the only works related to our method in which the focus is in the change of the sets of states and actions, not only in the transition and reward functions, nevertheless comparison with these methods is not possible since their goals and results differs from \ASPRL{}.

\section{Conclusion}\label{sec:conc} 

This paper presented a method for efficiently solving non-stationary Markov Decision Processes (MDP). The proposed approach, called \ASPRL{}, uses a combination of Answer Set Programming (ASP) and Reinforcement Learning (RL) in which ASP provides the set of states and actions in domains where unforeseen changes may happen, while RL is used to approximate a value-action function by means of interactions with the environment. In \ASPRL{}, Answer Set Programming is used as a tool for reasoning and knowledge revision and Reinforcement Learning allows for learning the solution of an MDP without the need of an explicit stationary reward function.

Experiments were performed in a changing grid world, whose results show that the use of ASP to find the set of states and actions effectively reduces the search space for finding optimal policies of  Markov Decision Processes in complex domains, as well as in domains that allow only a few possible policies. Not only \ASPRL{} allowed a faster approximation of the action-value function (compared to standard RL algorithms), but the process could continue to interact in a changing environment indefinitely.

\ASPRL{} is capable of dealing with unforeseen changes in the domain, thus solving non-stationary decision making problems. To the best of our knowledge, this has never been accomplished before.

Future work shall be directed towards a full integration of RL into the ASP engine, facilitating the use of ASP  when new states appear in a non-deterministic environment, with the possibility of reviewing the whole set of states seamlessly.

\begin{acknowledgements}
    Leonardo A. Ferreira was partially funded by CAPES. Reinaldo A. C. Bianchi acknowledges the support of FAPESP (grants 2011/19280-8 and 2016/21047-3). Paulo E. Santos acknowledges the support of CNPq (grants 307093/2014-0 and 473989/2013-1). Ramon Lopez de Mantaras acknowledges the support of Generalitat de Catalunya (project 2014-SGR-118) and CSIC (project NASAID 201550E022).
\end{acknowledgements}

\bibliographystyle{spmpsci}
%\bibliographystyle{spphys}

%\bibliography{references}

\end{document}